\documentclass[11pt]{article}
\usepackage[preprint]{acl}
\usepackage{times}
\usepackage{latexsym}
\usepackage[T1]{fontenc}
\usepackage[utf8]{inputenc}
\usepackage{microtype}
\usepackage{inconsolata}

\usepackage{algorithm}
\usepackage{algorithmic}
\usepackage{booktabs}       % professional-quality tables
\usepackage{amsfonts}       % blackboard math symbols
\usepackage{nicefrac}       % compact symbols for 1/2, etc.
\usepackage{xcolor}         % colors
\usepackage{graphicx}
\usepackage{subfigure}
\usepackage{amsmath}
\usepackage{amssymb}
\usepackage{mathtools}
\usepackage{amsthm}
\usepackage{tcolorbox}
\usepackage{colortbl}
\usepackage{pifont} % For check and cross marks
\usepackage{multirow}
\usepackage[capitalize,noabbrev]{cleveref}
\usepackage{enumitem}

\theoremstyle{plain}
\ifx\theorem\undefined
\newtheorem{theorem}{Theorem}

\newtheorem{proposition}{Proposition}

\theoremstyle{definition}
\newtheorem{definition}{Definition}

\DeclareRobustCommand\onedot{\futurelet\@let@token\@onedot}
\def\eg{\emph{e.g., }}

\title{LADR: Locality-Aware Dynamic Rescue for Efficient Text-to-Image Generation with Diffusion Large Language Models}

\author{
\textbf{Chenglin Wang\textsuperscript{1,*}},
\textbf{Yucheng Zhou\textsuperscript{2,*}},
\textbf{Shawn Chen\textsuperscript{3}},
\textbf{Tao Wang\textsuperscript{4}},
\textbf{Kai Zhang\textsuperscript{1, $\dagger$}}
\\
 \textsuperscript{1}East China Normal University,
 \textsuperscript{2}University of Macau \\
 \textsuperscript{3}Zhejiang University,
 \textsuperscript{4}Nanjing University
\\
\texttt{52275901013@stu.ecnu.edu.cn},~~
\texttt{yucheng.zhou@connect.um.edu.mo} \\
\texttt{kzhang980@gmail.com}
}

\begin{document}
\maketitle
\begingroup
\renewcommand\thefootnote{}
\footnotetext{$^*$ Equal Contributions.}
\footnotetext{$^\dagger$ Corresponding Author.}
\endgroup
\begin{abstract}
Discrete Diffusion Language Models have emerged as a compelling paradigm for unified multimodal generation, yet their deployment is hindered by high inference latency arising from iterative decoding. Existing acceleration strategies often require expensive re-training or fail to leverage the 2D spatial redundancy inherent in visual data. To address this, we propose \textbf{Locality-Aware Dynamic Rescue (LADR)}, a training-free method that expedites inference by exploiting the spatial Markov property of images. LADR prioritizes the recovery of tokens at the ``generation frontier'', regions spatially adjacent to observed pixels, thereby maximizing information gain. Specifically, our method integrates morphological neighbor identification to locate candidate tokens, employs a risk-bounded filtering mechanism to prevent error propagation, and utilizes manifold-consistent inverse scheduling to align the diffusion trajectory with the accelerated mask density. Extensive experiments on four text-to-image generation benchmarks 
demonstrate that our LADR achieves an approximate \textbf{4$\times$ speedup} over standard baselines. Remarkably, it maintains or even enhances generative fidelity, particularly in spatial reasoning tasks, offering a state-of-the-art trade-off between efficiency and quality.
\end{abstract}

\section{Introduction}
The field of generative modeling has witnessed a paradigm shift with the rapid evolution of Discrete Diffusion Language Models (DLMs)~\cite{sahoo2024simple,nie2025large,xin2025lumina}. Unlike Autoregressive (AR) models~\cite{radford2018improving, touvron2023llama,achiam2023gpt,zhoucondition,songbroad,zhou2025medical,zhou2024less}
that generate sequences strictly left-to-right, or Continuous Diffusion Models~\cite{ho2020denoising,rombach2022high,liuflow,yang2025self} that operate in continuous latent or pixel space, DLMs formulate visual generation~\cite{wang2025complexbench,zhou2025draw,esser2021taming} as a bidirectional masked modeling task within a discretized vector-quantized (VQ) latent space. This paradigm not only enables flexible, non-sequential generation orders but also facilitates unified multimodal understanding and generation within a single framework~\cite{li2025lavida,you2025llada}. By treating visual patches as discrete tokens akin to text, DLMs have achieved impressive scalability and fidelity, emerging as a powerful competitor to traditional paradigms.

\begin{figure}[t]
    \centering
    \includegraphics[width=0.99\linewidth]{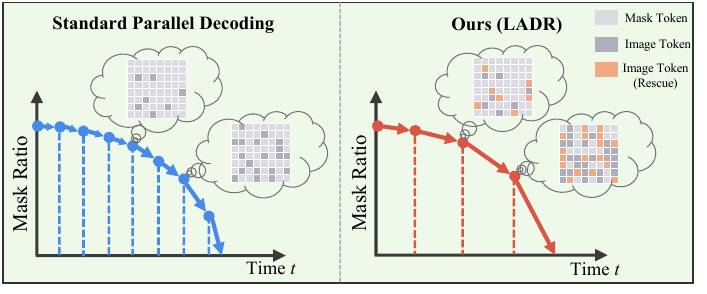}
    \caption{Comparison between Standard Parallel Decoding and our LADR method.
    While standard parallel decoding follows a fixed schedule, LADR 
    accelerates decoding by exploiting spatial locality to dynamically recover neighbor tokens and keeps generation quality.}
    \label{fig:intro}
    \vspace{-2mm}
\end{figure}

However, the iterative nature of DLMs imposes a severe bottleneck on inference efficiency. High-fidelity generation typically requires 50 to 100 forward passes to progressively refine the noisy sequence. Unlike AR models that benefit from KV-caching mechanisms to reuse historical computations, masked diffusion models must re-calculate bidirectional attention interactions at every step. While acceleration techniques exist, they often fall short in practicality: distillation-based methods~\cite{zhu2025di, zhu2025soft} require computationally expensive re-training and student-teacher alignment, limiting their flexibility. On the other hand, heuristic strategies borrowed from textual Masked Language Models (MLMs)~\cite{li2025diffusion,ye2025dream} often fail to generalize to the visual domain, as they overlook the fundamental difference between 1D textual dependencies and 2D visual structures.

Our work addresses this inefficiency by exploiting a property intrinsic to images but largely ignored in standard parallel decoding: \textit{Spatial Locality}. 
As illustrated in Fig.~\ref{fig:intro}, standard decoding schedules~\cite{chang2022maskgit, you2025llada} (\eg Cosine) assume isotropic uncertainty reduction, treating all masked tokens as independent variables. In contrast, we observe that images exhibit a strong spatial Markov property, the uncertainty of a pixel is significantly reduced if its immediate spatial neighbors are known. Based on this insight, we hypothesize that the most efficient decoding path is not random, but topological. By prioritizing the ``generation frontier'', the unmasked tokens spatially adjacent to observed regions, we can accelerate the transition from noise to structure.  

To materialize this insight, we propose \textbf{Locality-Aware Dynamic Rescue (LADR)}, a training-free acceleration method tailored for discrete visual generation. LADR dynamically modifies the decoding trajectory through three coupled mechanisms. First, it employs morphological operations to identify the topological neighbors of the current generation frontier. Second, to prevent the ``hallucination'' risks associated with aggressive acceleration, we introduce a risk-bounded filtering mechanism derived from the confidence gap of the model's posterior. Finally, to address the distribution shift caused by rapid mask reduction, we devise a \textit{Manifold-Consistent Inverse Scheduling} strategy that re-aligns the diffusion timesteps with the actual mask density, ensuring the denoiser operates within its trained support.

We extensively validate LADR on multiple comprehensive benchmarks, including GenEval~\cite{ghosh2023geneval},  UniGenBench~\cite{wang2025pref}, DPG-Bench~\cite{hu2024ella}, and T2I-CompBench~\cite{huang2023t2i}. Experimental results demonstrate that LADR significantly outperforms standard sampling and existing acceleration baselines. Notably, our method achieves an approximate \textbf{4$\times$ speedup} (reducing inference time from $\sim$57s to $\sim$13s) without compromising generative quality. In tasks requiring spatial reasoning (\eg object positioning and counting), LADR even surpasses the baseline performance, suggesting that enforcing spatial contiguity during decoding acts as a beneficial inductive bias.

In summary, our contributions are as follows:
\begin{itemize}[leftmargin=*, itemsep=1pt, topsep=1pt, partopsep=1pt, parsep=1pt]
    \item We identify \textit{spatial locality} as a critical but underutilized source of information gain in discrete diffusion, theoretically showing that topology-aware decoding minimizes conditional entropy more effectively than random selection.
    \item We propose \textbf{LADR}, a plug-and-play acceleration method that integrates morphological neighbor identification, theoretically grounded risk filtering, and inverse scheduling to safely expedite inference without re-training.
    \item We achieve state-of-the-art efficiency-quality trade-offs on widely adopted benchmarks, demonstrating that LADR can accelerate large-scale multimodal DLMs by $4\times$ while maintaining robust semantic alignment and visual fidelity.
\end{itemize}

\section{Related Work}
\label{sec:related_work}
Discrete Diffusion Language Models (DLMs) cast image generation as iterative masked token recovery in a discretized VQ space, enabling parallel decoding and substantially fewer sampling steps than continuous diffusion~\cite{sahoo2024simple,nie2025large,song2025seed,arriola2025block,ho2020denoising,rombach2022high,chen2025towards,yang2025hicogen,yang2025dc}. Initiated by MaskGIT~\cite{chang2022maskgit}, this framework has been extended by Paella~\cite{rampas2022novel} and Muse~\cite{chang2023muse} to improve robustness and semantic control, and more recently generalized to unified multimodal generation by modeling visual and textual tokens as a single sequence~\cite{you2025llada,swerdlow2025unified,xin2025lumina,li2025lavida}. Despite these advances, masked discrete diffusion remains latency-bound due to its reliance on iterative refinement with bidirectional attention and dynamically changing masks, which precludes computation reuse and contrasts sharply with KV-cached autoregressive decoding~\cite{li2024snapkv,bai2023qwen,guo2025deepseek,cai2024pyramidkv}. Distillation-based acceleration methods compress multi-step diffusion trajectories~\cite{hinton2014distilling,SongD0S23,deschenaux2025beyond}, but adapting consistency-style objectives to discrete VQ spaces is non-trivial and typically requires expensive retraining or relaxation techniques~\cite{zhu2025di,zhu2025soft}. In parallel, training-free acceleration heuristics have shown promise in text diffusion and sequence models~\cite{wu2025fast2,hu2025accelerating,wu2025fast,wang2025diffusion,li2025diffusion,israel2025accelerating}, yet their direct transfer to image generation remains limited, as they fail to explicitly exploit the strong 2D spatial locality inherent in visual tokens. The extended version can be found in Appendix~\ref{app:related_work}.

\begin{figure*}[!t]
\centering
\includegraphics[width=0.99\linewidth]{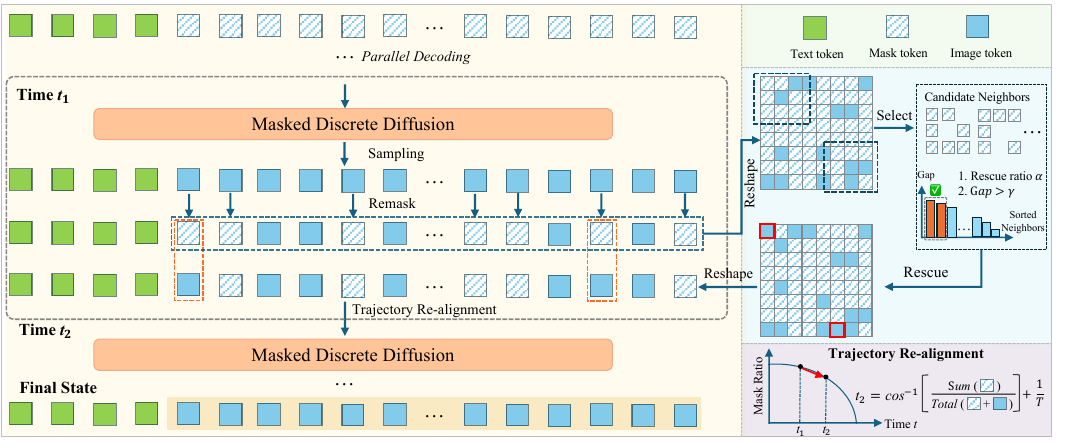}
\caption{Overview of the LADR method. At each timestep, the flattened discrete tokens were reshaped into a 2D grid to identify candidate neighbors adjacent to resolved regions. These candidates are evaluated using the \textit{Confidence Margin} (confidence top1-top2 gap) and are dynamically "rescued" (unmasked) based on an adaptive rescue ratio $\alpha$ and threshold $\gamma$.  To synchronize the generation timeline with this accelerated accumulation of tokens, the \textit{Trajectory Re-alignment} module utilizes an inverse cosine function to re-calculate the effective timestep $t_2$, allowing the scheduler to skip redundant iterations while maintaining consistency.}
\label{figure-1}
\end{figure*}

\section{Methodology}
\label{sec:method}

As illustrated in Fig.~\ref{figure-1}, we proposed \textbf{Locality-Aware Dynamic Rescue (LADR)}, a method designed to accelerate Discrete Diffusion Language Models (DLMs) while preserving generation quality. 
Our approach is grounded in the observation that standard parallel decoding typically treats tokens as independent variables given the global context. However, image representations derived from convolutional encoders inherently exhibit strong \textit{spatial locality}.
In this section, we first formalize the generation process and then provide the theoretical motivation grounded in information theory and risk estimation to drive our three algorithmic components: morphological neighbor identification, risk-bounded filtering, and manifold-consistent inverse scheduling.

\subsection{Preliminaries: Discrete Diffusion and Variational Bounds}
\label{subsec:prelim}
Let $\mathbf{z}_0 = [z_{0,1}, \dots, z_{0,N}] \in \mathcal{V}^N$ represent a discrete image sequence flattened from a $H \times W$ feature map, where each token belongs to a codebook $\mathcal{V}$. The discrete diffusion process is a forward Markov chain $q(\mathbf{z}_t | \mathbf{z}_{t-1})$ that progressively corrupts $\mathbf{z}_0$ by replacing tokens with a special $[\texttt{MASK}]$ token. The marginal distribution at time $t \in [0, 1]$ is given by:
\begin{align}
    &q(\mathbf{z}_t | \mathbf{z}_0) = \prod_{i=1}^N q(z_{t,i} | z_{0,i}), \notag\\
    &\text{where } q(z_{t,i} = \texttt{[MASK]} | z_{0,i}) = \gamma(t),
    \label{eq:forward_process}
\end{align}
where $\gamma(t)$ is a monotonic masking schedule (\eg cosine) representing the probability of a token being masked. The reverse process $p_\theta(\mathbf{z}_{0} | \mathbf{z}_t)$ approximates the true posterior $q(\mathbf{z}_0 | \mathbf{z}_t)$. The training objective is to minimize the negative Evidence Lower Bound (ELBO), which simplifies to the negative log-likelihood over masked regions $\mathcal{M}_t$:
\begin{align}
    \mathcal{L} \approx \mathbb{E}_{t, \mathbf{z}_0} \left[ - \sum_{i \in \mathcal{M}_t} \log p_\theta(z_{0,i} | \mathbf{z}_{t, \mathcal{O}_t}) \right],
    \label{eq:elbo_loss}
\end{align}
where $\mathcal{O}_t$ denotes the set of observed indices. During inference, iterative decoding approximates the joint distribution via conditional independence assumption: $p_\theta(\mathbf{z}_0 | \mathbf{z}_t) \approx \prod_{i \in \mathcal{M}_t} p_\theta(z_{0, i} | \mathbf{z}_{t, \mathcal{O}_t})$. Standard acceleration methods strictly follow $\gamma(t)$, discarding potentially correct predictions in early stages.

\subsection{Theoretical Motivation}
\label{subsec:theory}
Instead of relying on heuristic acceleration, we formulate LADR by analyzing the entropy reduction and risk bounds within the discrete latent space.

\subsubsection{Entropy Reduction via Local Information Gain}
Standard decoding assumes isotropic uncertainty reduction. However, since discrete image tokens $\mathbf{z}$ are typically obtained via CNN-based encoders (\eg VQGAN), the dependency between tokens decays with spatial distance due to bounded \textit{Effective Receptive Fields (ERFs)}.
We quantify the uncertainty of a masked token $z_i$ using Conditional Entropy $H(z_i | \mathbf{z}_{\mathcal{O}})$. The reduction in uncertainty gained by observing an auxiliary set $\mathcal{S}$ is quantified by the Conditional Mutual Information:
\begin{align}
    I(z_i; \mathbf{z}_{\mathcal{S}} | \mathbf{z}_{\mathcal{O}}) = H(z_i | \mathbf{z}_{\mathcal{O}}) - H(z_i | \mathbf{z}_{\mathcal{O}}, \mathbf{z}_{\mathcal{S}}).
\end{align}

\begin{definition}[Generation Frontier]
Given a binary mask $\mathbf{M}$, the generation frontier $\mathcal{F}$ is defined as the set of masked tokens spatially adjacent to currently observed tokens:
$\mathcal{F} = \{ i \mid m_i = 1 \land \exists j \in \mathcal{N}(i), m_j = 0 \}$, where $\mathcal{N}(i)$ is the local spatial neighborhood.
\end{definition}

\begin{proposition}[Locality-Driven Information Lower Bound]
\label{prop:entropy}
Given the spatial inductive bias of the encoder, the mutual information between a token $z_i$ and its immediate neighborhood $\mathcal{N}(i)$ dominates that of distant context $\mathcal{S}_{dist}$. Formally:
\begin{align}
    I(z_i; \mathbf{z}_{\mathcal{N}(i)} | \mathbf{z}_{\mathcal{O}}) \gg I(z_i; \mathbf{z}_{\mathcal{S}_{dist}} | \mathbf{z}_{\mathcal{O}}).
\end{align}
\end{proposition}
\textit{Remark.}
This proposition provides the theoretical justification for our \textbf{Morphological Neighbor Identification} strategy: by prioritizing the generation frontier $\mathcal{F}$, LADR maximizes the expected information gain per decoding step, guiding sampling along the local
structure of the latent manifold. This locality assumption is also empirically illustrated in Fig.~\ref{fig:figure3}, where perturbing a small number of VQ tokens induces only spatially confined changes in the decoded image, while the global structure remains largely intact.

\subsubsection{Safety Guarantee via Margin Bounds}
Accelerating generation involves ``rescuing'' tokens before their scheduled timestamp. To control the quality, we must bound the probability of misclassification. We employ the \textit{Confidence Gap}, $\Delta_i = p_{(1)} - p_{(2)}$, where $p_{(1)}$ and $p_{(2)}$ are the top-1 and top-2 probabilities.

\begin{theorem}[Margin-based Error Bound]
\label{thm:risk_bound}
Consider a classification task over $K$ classes. If the predicted distribution satisfies a confidence margin $\Delta \ge \tau$, the probability of error $P(\mathcal{E})$ is strictly upper bounded. Specifically, in the worst-case distribution scenario:
\begin{align}
    P(\mathcal{E}) \le 1 - \left( \frac{1 + \tau}{2} \right).
    \label{eq:risk_bound}
\end{align}
\end{theorem}
\textit{Proof.} See Appendix~\ref{proof:thm1}.

\textit{Remark.} This theorem provides a controllable mechanism. By enforcing a dynamic threshold $\tau(t)$, we can theoretically bound the error rate of our acceleration module, ensuring that the rescued tokens satisfy a minimum reliability standard.

\subsubsection{Manifold Consistency via Inverse Scheduling}
A critical challenge in acceleration is the \textbf{Training-Inference Mismatch}. Aggressive rescue reduces the mask ratio $\rho_{act}$ faster than the scheduler $\gamma(t)$ expects, potentially pushing the state out-of-distribution (OOD).

\begin{proposition}[Manifold Consistency Condition]
\label{prop:manifold}
To ensure the input state remains within the support of the trained diffusion manifold, the conditioning timestep $t$ must be re-aligned such that the expected mask density matches the actual observation density:
\begin{align}
    t_{new} = \gamma^{-1}(\rho_{act}) \quad \text{s.t.} \quad \mathbb{E}[\rho(t_{new})] \approx \rho_{act}.
\end{align}
\end{proposition}
\textit{Remark.} This necessitates our \textbf{Inverse Scheduling} technique, which acts as a temporal projection operator to correct the trajectory after aggressive rescue operations. $\gamma^{-1}$ is the inverse function of $\gamma$.

\begin{figure}
    \centering
    \includegraphics[width=0.99\linewidth]{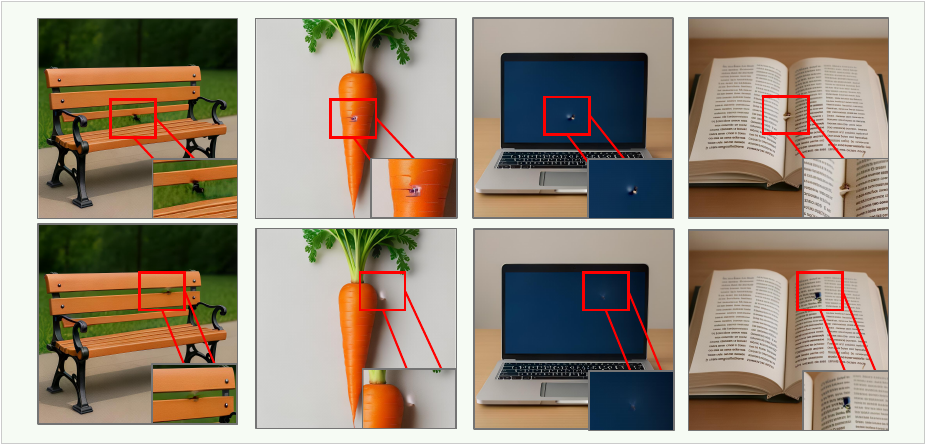}
    \caption{Visualization of localized semantic changes caused by perturbing a small set of VQ tokens. The impact remains spatially confined, supporting the locality assumption that nearby tokens dominate information gain.
}
    \label{fig:figure3}
\end{figure}

\subsection{The LADR Method}
Guided by the theoretical motivations above, LADR dynamically updates the mask tokens $\mathbf{M}_t$ at each timestep $t$ through three coupled steps.

\subsubsection{Morphological Neighbor Identification}
Leveraging Proposition~\ref{prop:entropy}, we aim to identify the frontier $\mathcal{F}$. We map the 1D mask sequence to the 2D spatial grid $\Phi: \{0, 1\}^N \to \{0, 1\}^{H \times W}$. Using a spatial kernel $\mathbf{K}$ (\eg $3 \times 3$), the candidate neighbors $\mathcal{C}_t$ are identified via morphological dilation:
\begin{align}
    \mathbf{M}_{\text{grid}} &= \Phi(\mathbf{M}_t), \\
    \mathbf{M}_{\text{frontier}} &= \mathbf{M}_{\text{grid}} \land (\neg \mathbf{M}_{\text{grid}} \oplus \mathbf{K}), \\
    \mathcal{C}_t &= \{ i \mid \Phi^{-1}(\mathbf{M}_{\text{frontier}})[i] = 1 \}.\label{eq9}
\end{align}
This explicitly selects masked tokens that share spatial connectivity with observed regions.

\subsubsection{Phase-Aware Dynamic Filtering}
For every candidate $i \in \mathcal{C}_t$, we compute the confidence gap $\Delta_i$. Guided by Theorem~\ref{thm:risk_bound}, we employ a dynamic policy $\Pi(t) = (\alpha_t, \tau_t)$ that adapts to the entropy of the generation phase defined by the effective timestep $t_{eff}$:
\begin{itemize}[leftmargin=*]
    \item \textbf{Exploration Phase ($t_{eff} < 0.2$):} Global entropy is high. We apply a strict threshold $\tau=0.05$ and limit the rescue ratio $\alpha=0.1$ to prevent error propagation.
    \item \textbf{Structure Phase ($0.2 \le t_{eff} < 0.7$):} As semantics emerge, we relax constraints ($\tau=0.05, \alpha=0.3$).
    \item \textbf{Refinement Phase ($t_{eff} \ge 0.7$):} We aggressively rescue neighbors ($\tau=\emptyset, \alpha=1.0$) to fill texture details.
\end{itemize}
The set of rescued tokens $\mathcal{R}_t$ is:
\begin{align} \label{eq10}
    \mathcal{R}_t = \text{TopK}_{\Delta}\left( \{ i \in \mathcal{C}_t \mid \Delta_i > \tau_t \}, \lfloor |\mathcal{C}_t| \cdot \alpha_t \rfloor \right).
\end{align}

\subsubsection{Trajectory Re-alignment}
After unmasking $\mathcal{R}_t$, the sequence sparsity decreases to $\rho_{new}$. Crucially, continuing with the original schedule $t$ would violate the manifold consistency (Proposition~\ref{prop:manifold}). We thus re-calculate the next sampling step $t_{next}$ using the inverse schedule:
\begin{align}
    t_{next} &= \text{clamp}\left( \gamma^{-1}(\rho_{new}) + \frac{1}{T}, 0, 1 \right).
\end{align}
This adjustment ensures that the noise level estimates remain accurate, effectively ``skipping'' redundant diffusion steps. The complete procedure is summarized in Algorithm~\ref{alg:ladr}.

\begin{algorithm}[t]\small
\caption{\small Locality-Aware Dynamic Rescue (LADR)}
\label{alg:ladr}
\begin{algorithmic}[1]
\REQUIRE Pre-trained DLM $p_\theta$, Scheduler $\gamma(\cdot)$, Steps $T$
\ENSURE Discrete tokens $\mathbf{z}$
\STATE \textbf{Initialize:} $\mathbf{z} \leftarrow [\texttt{MASK}]^N$, $\mathbf{M} \leftarrow \mathbf{1}^N$
\STATE $N_{total} \leftarrow N$
\FOR{$step = 0$ \TO $T-1$}
    \IF{$\sum \mathbf{M} = 0$} \STATE \textbf{break} \ENDIF
    
    \STATE \COMMENT{\textcolor{blue}{\textit{Step 1: Inverse Scheduling (Prop. 1)}}}
    \STATE $\rho_{curr} \leftarrow (\sum \mathbf{M}) / N_{total}$
    \STATE $t_{eff} \leftarrow \gamma^{-1}(\rho_{curr})$ \quad \textcolor{gray}{// Align time with mask density}\!\!\!\!\!
    \STATE $t_{next} \leftarrow \text{clamp}(t_{eff} + 1/T, 0, 1)$
    \STATE $n_{mask} \leftarrow \lfloor N_{total} \cdot \gamma(t_{next}) \rfloor$ \quad \textcolor{gray}{// Target mask count}
    
    \STATE \COMMENT{\textcolor{blue}{\textit{Step 2: Parallel Prediction}}}
    \STATE $\mathbf{L} \leftarrow p_\theta(\mathbf{z}, t_{next})$
    \STATE $\mathbf{P} \leftarrow \text{Softmax}(\mathbf{L})$
    \STATE $\mathbf{z}_{pred} \leftarrow \text{Sampling}(\mathbf{P})$
    \STATE $\boldsymbol{\Delta} \leftarrow \text{Top1}(\mathbf{P}) - \text{Top2}(\mathbf{P})$
    
    \STATE \COMMENT{\textcolor{blue}{\textit{Step 3: Standard Selection (Global)}}}
    \STATE $\mathcal{I}_{rank} \leftarrow \text{Argsort}(\text{Top1}(\mathbf{P}) \cdot \mathbf{M}, \text{descending})$
    \STATE $\mathbf{M}_{std} \leftarrow \mathbf{1}^N$
    \STATE $\mathbf{M}_{std}[\mathcal{I}_{rank}[n_{mask}:N]] \leftarrow 0$ \quad \textcolor{gray}{// Unmask most confident}
    
    \STATE \COMMENT{\textcolor{blue}{\textit{Step 4: Neighbor Rescue (Lemma 1 \& Thm 1)}}}
    \STATE $\mathbf{M}_{grid} \leftarrow \text{Reshape}(\mathbf{M}_{std}, H, W)$
    \STATE $\mathbf{M}_{front} \leftarrow \mathbf{M}_{std} \land \text{Flatten}(\neg \mathbf{M}_{grid} \oplus \mathbf{K}_{3\times3})$
    \STATE $\mathcal{C}_{neigh} \leftarrow \{i \mid \mathbf{M}_{front}[i] = 1\}$
    
    \IF{$|\mathcal{C}_{neigh}| > 0$}
        \STATE Get $\alpha, \tau$ based on $t_{eff}$ (Sec 3.3.2)
        \STATE $\mathcal{C}_{valid} \leftarrow \{i \in \mathcal{C}_{neigh} \mid \boldsymbol{\Delta}[i] > \tau\}$
        \STATE $k_{res} \leftarrow \min(\lfloor |\mathcal{C}_{neigh}| \cdot \alpha \rfloor, |\mathcal{C}_{valid}|)$
        \STATE $\mathcal{S}_{res} \leftarrow \text{Argsort}(\boldsymbol{\Delta}[\mathcal{C}_{valid}], \text{descending})[0:k_{res}]$\!\!\!\!\!
        \STATE $\mathbf{M}_{std}[\mathcal{S}_{res}] \leftarrow 0$
    \ENDIF
    
    \STATE \COMMENT{\textcolor{blue}{\textit{Step 5: State Update}}}
    \STATE $\mathbf{M} \leftarrow \mathbf{M}_{std}$
    \STATE $\mathbf{z} \leftarrow \mathbf{z}_{pred} \odot (\mathbf{1} - \mathbf{M}) + [\texttt{MASK}] \odot \mathbf{M}$
\ENDFOR
\RETURN $\mathbf{z}$
\end{algorithmic}
\end{algorithm}

\section{Experiments}

\subsection{Experimental setup}
\paragraph{Benchmarks and Baselines.}
To strictly evaluate the effectiveness of our method on both inference efficiency and visual fidelity, we conducted evaluations across four publicly popular text-to-image generation benchmarks: \textbf{GenEval}~\cite{ghosh2023geneval}, \textbf{UniGen-Bench}~\cite{wang2025pref}, \textbf{DPG-Bench}~\cite{hu2024ella}, and \textbf{T2I-CompBench}~\cite{huang2023t2i}. These benchmarks provide a comprehensive assessment spanning from basic object semantics to complex compositional generation. 
Furthermore, to ensure a fair and focused evaluation, we compared our method against two representative training-free acceleration method: (1) \textbf{ML-Cache}~\cite{xin2025lumina}, the native caching optimization strategy embedded in the Lumina-DiMOO backbone; and (2) \textbf{Prophet}~\cite{li2025diffusion}, a heuristic-based accelerated decoding method originally designed for text generation, which we adapted to the visual domain to investigate the cross-modal applicability of textual heuristics.

\paragraph{Implementation Details.}
For a fair and controlable comparison, 
we adopt the unified multimodal model Lumina-DiMOO~\cite{xin2025lumina} as the foundational DLM backbone for all experiments since comparable open-source models are limited.  The generated image resolution is $1024\times 1024$. To ensure a consistent evaluation of inference latency, all models and baselines were executed locally on a single NVIDIA A100 (80GB) GPU. 
We adhered to the standard inference configurations of the backbone model, reporting the performance following the evaluation scripts of each benchmark. Notably, for the UniGen-Bench, we used the version of their released scripts about the open-source vision-language model Qwen2.5-VL-72B~\cite{qwen2.5-VL} to evaluate.
Follow Prophet~\cite{li2025diffusion}, we divided the parallel decoding process into three phases, and set phase-aware thresholds to regulate the rescued neighbors in the decoding process. Much like it established a proof-of-concept for heuristic-based acceleration in text decoding, our work aims to pioneer a similar trajectory for visual decoding. Consequently, we did not perform an exhaustive grid search to obtain these parameters for each specific dataset.

\subsection{Main Results}
\begin{table*}[!t]
    \centering
    \renewcommand{\arraystretch}{1.2}
    \resizebox{\textwidth}{!}{
    \begin{tabular}{l c cccccc c}
        \toprule
        \textbf{Method} & \textbf{Avg. t (s)}$\downarrow$ & \textbf{Two Obj.} & \textbf{Colors} & \textbf{Attribute} & \textbf{Single Obj.} & \textbf{Position} & \textbf{Counting} & \textbf{Overall $\uparrow$} \\
        \cmidrule(lr){1-9}

        Lumina-DiMOO & 57.01 & \textbf{93.94} & \textbf{91.49} & 73.00 & ~~97.50 & 79.00 & \underline{85.00} & 86.66 \\
        \cmidrule(lr){1-9}
        ~~~~+ ~ML-Cache    & \underline{31.95} & \underline{93.75} & 89.63 & \textbf{75.50} & \textbf{100.00} & \underline{84.50} & \textbf{85.94} & \textbf{87.83} \\
        ~~~~+ ~Prophet & 32.15 & 91.16 & 86.44 & 70.75 & ~~96.56 & 74.50 & 84.06 & 83.91 \\
        ~~~~+ ~\textbf{LADR(Ours)}   & \textbf{13.22} & 91.41 & \underline{91.22} & \underline{74.75} & ~~\underline{99.06} & \textbf{85.50} & 81.88 & \underline{87.30} \\
        \bottomrule
    \end{tabular}
    }
    \caption{Performance comparison on the \textbf{GenEval}~\cite{ghosh2023geneval}.
    }
    \label{tab:geneval_results}
\end{table*}

\begin{table*}[!t]
    \centering
    \renewcommand{\arraystretch}{1.2}
    \resizebox{\linewidth}{!}{
    \begin{tabular}{l ccccccccccc}
        \toprule
        \textbf{Method}& \textbf{Avg. t (s)}$\downarrow$ & \textbf{Style} & \textbf{Know.} & \textbf{Attr.} & \textbf{Action} & \textbf{Rel.} & \textbf{Cmp.} & \textbf{Gram.} & \textbf{Logic.} & \textbf{Lay.} & \textbf{Text} \\
        \cmidrule(lr){1-12}
        Lumina-DiMOO & 57.21 & \underline{91.52} & \textbf{89.87} & \underline{79.29} & 71.48 & 78.55 & 73.45 & \textbf{69.79} & 43.58 & \textbf{85.63} & \textbf{27.87}  \\
        \cmidrule(lr){1-12}
        ~~~~~+ ~ML-Cache & \underline{31.96} & 91.40 & \underline{88.77} & 79.17 & \underline{73.67} & \underline{79.06} & \underline{74.61} & \underline{69.65} & \textbf{45.87} & 83.96 & 26.15 \\
        ~~~~~+ ~Prophet & 32.47 & 87.40 & 84.34 & 75.64 & 68.25 & 75.63 & 65.85 & 64.97 & 39.91 & 83.02 & 25.57 \\
        ~~~~~+ ~LADR(\textbf{Ours}) & \textbf{13.94}  & \textbf{94.80} & 88.61 & \textbf{81.73} & \textbf{75.95} & \textbf{81.98} & \textbf{77.71} & 66.31 & \underline{45.41} & \underline{85.26} & 16.38 \\
        \bottomrule
    \end{tabular}
    }
    \caption{
    Performance comparison on \textbf{UniGen-Bench}~\cite{wang2025pref}.
    }
    \label{tab:unibench_results}
\end{table*}

\begin{table*}[!t]
    \centering
    \scriptsize
    \renewcommand{\arraystretch}{1.2}
    \resizebox{0.99\linewidth}{!}{%
    \begin{tabular}{l c cccccc c}
        \toprule
        \textbf{Method} & \textbf{Avg. t (s)} $\downarrow$ & \textbf{Color} & \textbf{Shape} & \textbf{Texture} & \textbf{Spatial} & \textbf{Non-spatial} & \textbf{Complex} \\
        \cmidrule(lr){1-8}
        Lumina-DiMOO & 56.92  & 81.07  & 57.02  & 69.30  & \underline{46.70}  & \underline{31.70}  & 34.98 \\
        \cmidrule(lr){1-8}
        ~~~~~+ ~ML-Cache & \underline{31.78}  & \underline{81.52}  & \underline{57.75}  & 70.28 & \textbf{46.79} & \textbf{31.83} & \underline{35.23}  \\
        ~~~~~+ ~Prophet & 32.43 & 80.92 & 55.48 & \underline{70.34} & 42.37 & 31.46 & 34.92 \\
        ~~~~~+ ~LADR(\textbf{Ours}) & \textbf{13.41} & \textbf{82.25} & \textbf{58.94} & \textbf{72.34}  & 46.35 & 31.60 & \textbf{36.19}  \\
        \bottomrule
    \end{tabular}
    }
    \caption{Performance Comparison on \textbf{T2I-CompBench}~\cite{huang2023t2i}.
     }
    \label{tab:t2i_compbench}
\end{table*}

\begin{table*}[!t]
    \centering
    \small
    \renewcommand{\arraystretch}{1.2}
    \resizebox{0.95\linewidth}{!}{
    \begin{tabular}{l c cccccc}
        \toprule
        \textbf{Method} & \textbf{Avg. t (s)} $\downarrow$ & \textbf{Global} & \textbf{Entity} & \textbf{Attribute} & \textbf{Relation} & \textbf{Other} & \textbf{Overall} $\uparrow$ \\
        \cmidrule(lr){1-8}
        Lumina-DiMOO & 58.16 & 77.20 & 90.36 & 87.93 & 93.04 & 82.80 & 83.61 \\
        \cmidrule(lr){1-8}
        ~~~~~+ ~ML-Cache   & \underline{32.01} & 81.46 & \underline{90.37} & \underline{88.16} & \underline{93.27} & \textbf{84.40} & \underline{84.02} \\
        ~~~~~+ ~Prophet & 34.63  & \underline{81.76} & 89.56 & 87.33 & 92.76 & \underline{83.20} & 82.91 \\
        ~~~~~+ ~LADR(\textbf{Ours})        & \textbf{14.52}     & \textbf{84.19} & \textbf{91.47} & \textbf{89.12} & \textbf{94.20} & 81.20 & \textbf{85.42} \\
        \bottomrule
    \end{tabular}
    }
    \caption{Performance evaluation on \textbf{DPG-Bench}~\cite{hu2024ella}.
    }
    \label{tab:dpg_bench}
\end{table*}

We empirically investigated the effectiveness of our proposed accelerated method LADR by answering two critical research questions: (1) \textit{Does the method deliver substantial speedups compared to existing caching and heuristic strategies?} (2) \textit{Can it maintain or even enhance generative fidelity in some scenarios?}

\paragraph{Decoding Efficiency Analysis.}
The primary motivation of our approach is to alleviate the computational bottleneck of parallel iterative decoding. As presented in Tables 1 through 4, our method demonstrates a dramatic reduction in inference latency across all benchmarks.
On average, our accelerated model completes inference in approximately $\sim$13-14 seconds, representing a $\mathbf{4\times}$ speedup over the No-Cache ($\sim$57s) and a $\mathbf{2\times}$ speedup over the optimized ML-Cache and Prophet ($\sim$32s). Notably, our method outperforms the text-optimized Prophet algorithm, confirming that our locality-aware rescue strategy is inherently more suitable for the 2D visual domain than heuristics transplanted from 1D text generation.

\paragraph{Generative Quality Analysis.}

Beyond efficiency, our results indicate that the significant reduction in sampling steps does not come at the cost of visual fidelity, while it yields a highly competitive performance profile across diverse benchmarks.
Conversely, we note a performance drop in fine-grained high-frequency tasks, such as the text rendering score in UniGen-Bench, suggesting that the model remains sensitive to the reduction of iterative refinement steps in some scenarios. To further investigate this, we conducted an additional experiment on UniGen-Bench's text category exploring a delayed rescue strategy. Specifically, we completely disabled the rescue operation during the Exploration Phase ($t_{eff} < 0.2$) and kept LADR’s other settings unchanged. As illustrated in Table~\ref{tab:delayed}, the text rendering score recovers to 26.10, which is comparable to the base model's 27.87. Meanwhile, the inference time is 22.51s, still maintaining a highly efficient 2.5x speedup. This demonstrates that LADR could achieve a highly adaptable efficiency-quality trade-off.
Overall, the experiment results demonstrate that our method achieves a superior efficiency-fidelity trade-off, delivering efficient decoding speeds while preserving robust generative capabilities in image generation.

\begin{table}[h]
\centering
\resizebox{0.98\linewidth}{!}{
\begin{tabular}{lcc}
\toprule
\textbf{Method} & {\textbf{Avg.t (s) $\downarrow$}} & {\textbf{Text}} \\
\midrule
Lumina-DiMOO             & 57.21 & 27.87 \\
LADR                     & 13.94 & 16.38 \\
LADR(delayed rescue)    & 22.51 & 26.10 \\
\bottomrule
\end{tabular}}
\caption{Efficiency and performance comparison of different strategies on the text category of UniGen-Bench.}
\label{tab:delayed}
\end{table}

\paragraph{Case Visualization.}  To intuitively assess the impact of acceleration on perceptual quality, 
We visualized some cases generated by our proposed LADR method alongside the base model and two training-free baselines. As shown in  Figure~\ref{fig-4:case}, our approach achieves a significant speedup ($ 4\times$ faster than the backbone) while preserving intricate visual details (\eg reflections in the "wooden boats") and correct semantic composition (\eg  "red cup and pink handbag"), demonstrating that our spatial-aware acceleration could maintain the generative quality of the underlying model.

\begin{figure}[!t]
    \centering
    \includegraphics[width=0.99\linewidth]{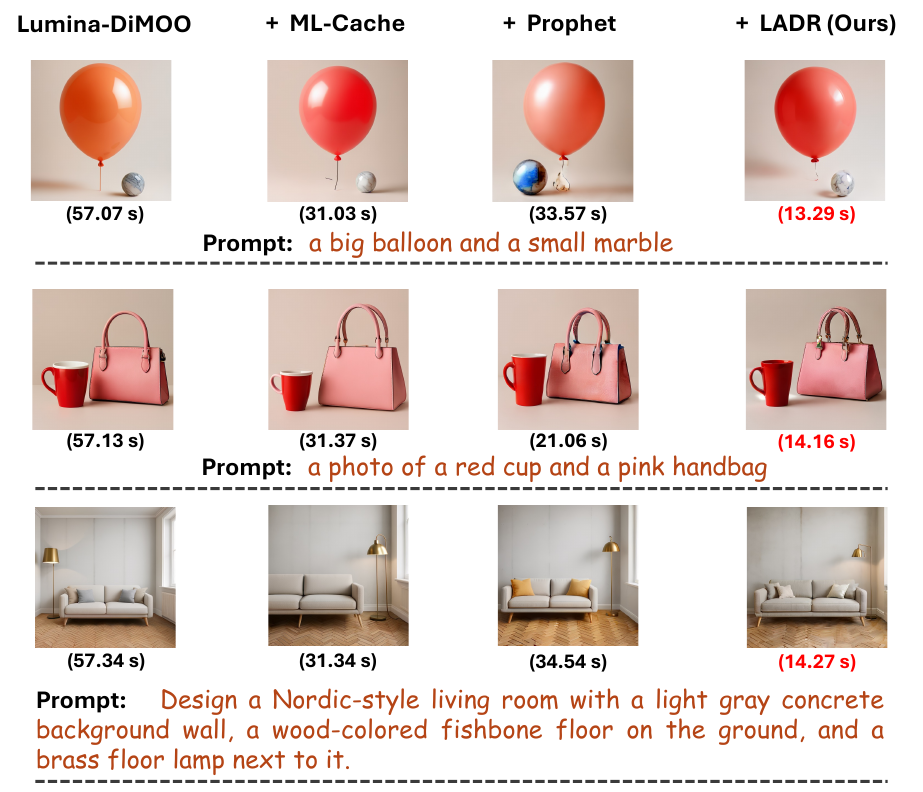}
    \caption{Qualitative comparison of different methods in terms of generation fidelity and inference time, where the corresponding text prompt is provided below each row, with the inference latency displayed in parentheses under the image. }
    \label{fig-4:case}
\end{figure}

\subsection{Ablation Studies and Analysis}
\paragraph{Impact of Spatial Selection Strategy.}

To strictly validate our hypothesis that spatial locality is the critical factor for acceleration, we conducted an ablation study on the token selection criteria. Specifically, we first determined the counts $k$ of rescued neighbor tokens via LADR at each timestep, and then enforced this exact budget on two strategies:

\begin{itemize}
    \item \textbf{Non-Neighbor Prioritization.} This strategy prioritizes isolated tokens with the top-$k$ confidence gaps, only reverting to neighbors if the non-neighbor set is exhausted.
    \item \textbf{Random Selection.} This strategy randomly sampled from the remask token, no better neighbor or non-neighbor tokens.
\end{itemize}

Figure \ref{fig:neighbor-non} presents the quantitative comparison on the GenEval benchmark.
We observe that the Non-Neighbor strategy yields the lowest performance (Overall score: 84.05), significantly lagging behind our method, which achieves an overall score of 87.30. This confirms that forcing the model to resolve isolated tokens early, even those with high confidence gaps, leads to error propagation, as these predictions may lack sufficient spatial grounding.
Interestingly, the random strategy outperforms the Non-Neighbor variant, but it is still lagging behind our proposed accelerated method.
This indicates that our locality-aware dynamic rescue strategy provides the optimal balance, ensuring that the accelerated decoding trajectory respects the structural dependencies of the image.

\paragraph{Effectiveness of Confidence Margin.}

Besides the necessity of spatial locality, we further investigate the optimality of the ranking metric used to filter these spatial candidates. To verify whether the \textit{Confidence Margin} (Top1-Top2 gap) provides a superior signal compared to standard confidence scores, we evaluate distinct prioritization criteria for selecting the rescued tokens $\mathcal{R}_{t}$ in eq~(\ref{eq10}):
\begin{itemize}
    \item \textbf{Standard Confidence (Top-1 Probability).} This variant ranks neighbors solely by the probability of the most likely token.
    \item \textbf{Random Neighbor Selection.} This baseline selects tokens stochastically from the neighborhood $\mathcal{C}_t$ in eq~(\ref{eq9}), ignoring predictive certainty entirely.
\end{itemize}

Figure~\ref{fig:conf} reports the performance on the GenEval benchmark.
The results demonstrate that our confidence margin strategy achieves the highest overall accuracy (87.30), surpassing the standard confidence baseline (85.24).
While the standard confidence approach performs strongly in object-centric metrics like \textit{Counting} (83.00), it underperforms in structural categories such as \textit{Position} (80.25 vs. 85.50) and \textit{Attribute} (71.25 vs. 74.75).
The Random neighbor selection yields the lowest overall performance (83.74).
These findings suggest that the Top1-Top2 gap is a more robust discriminator to recover tokens. It effectively penalizes ambiguously high predictions, where the model is confident in the top choice but equally confident in a competing alternative, thereby preventing the premature fixation of semantically unstable tokens.

\begin{figure}[t]
    \centering
    \includegraphics[width=\linewidth]{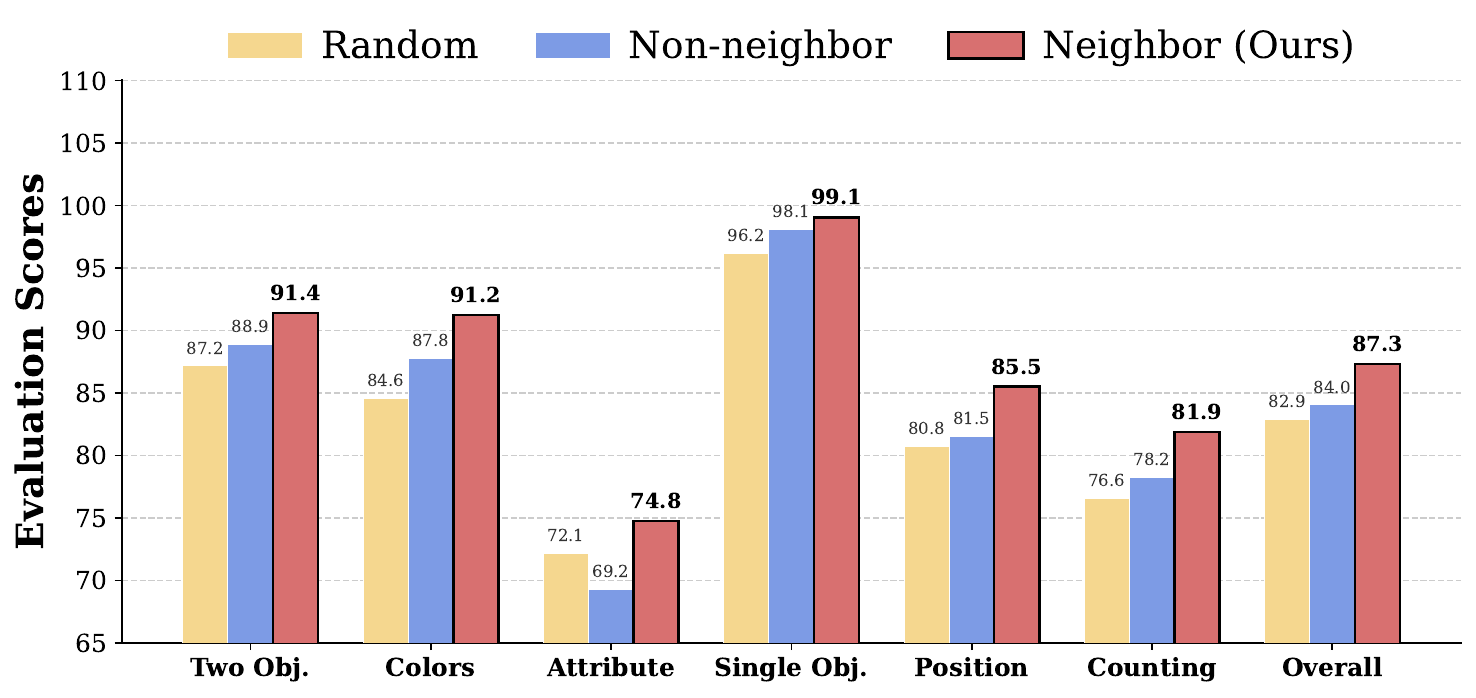}
    \caption{Ablation study of spatial selection strategies on the GenEval benchmark.``Random'' means \textit{Random Selection}, ``Non-neighbor'' represents \textit{Non-Neighbor Prioritization}, and ``Neighbor (Ours)'' is our strategy that rescues neighbor tokens.}
    \label{fig:neighbor-non}
\end{figure}

\begin{figure}[t]
    \centering
    \includegraphics[width=\linewidth]{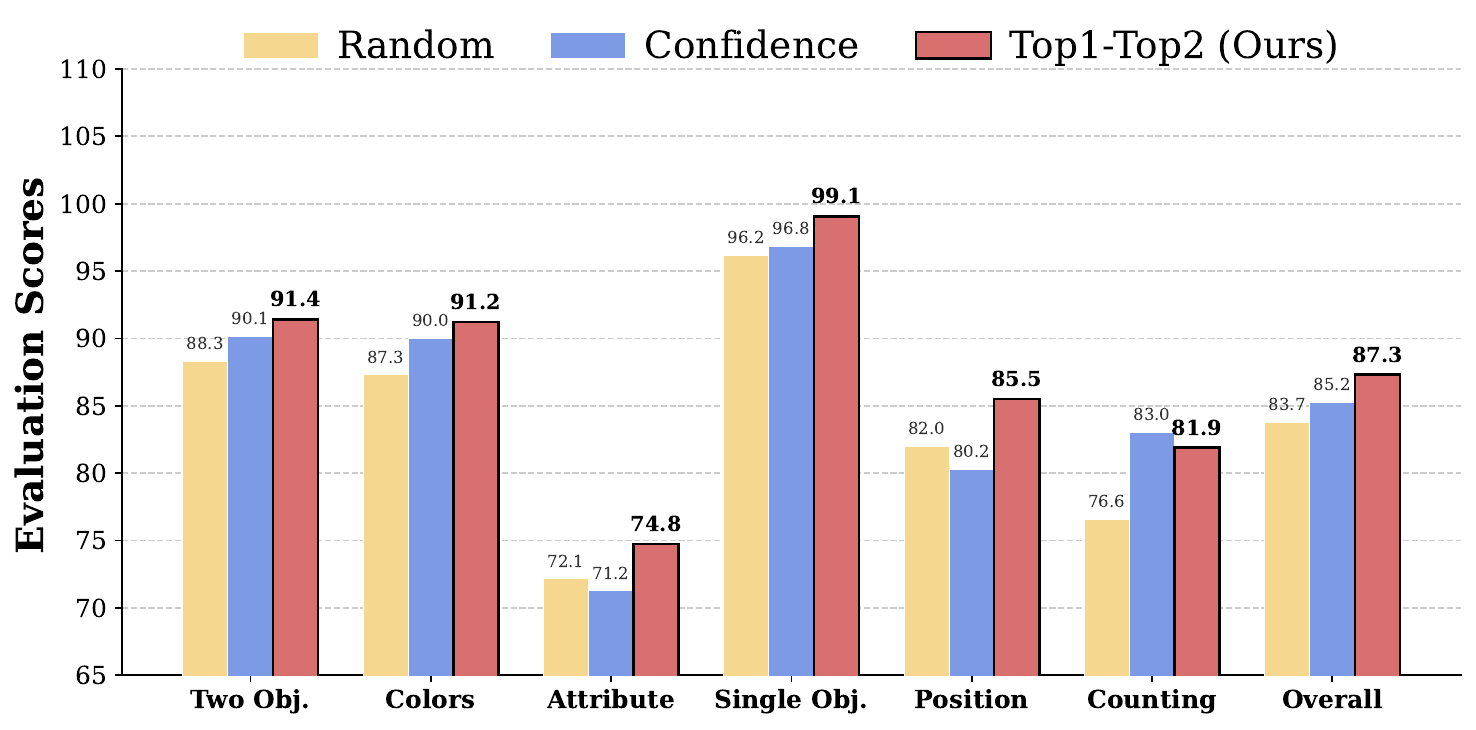}
    \caption{Ablation study about different ranking metrics of neighbor tokens on the GenEval benchmark.``Random'' means \textit{Random Neighbor Selection}, ``Confidence'' denotes \textit{Standard Confidence}, and ``Top1-Top2 (Ours)'' means confidence margin we employed. }
    \label{fig:conf}
\end{figure}

\section{Conclusion and Future Work}
In this paper, we propose an accelerated parallel decoding strategy called LADR, which is a training-free method designed to unlock the inference efficiency of DLMs. By challenging the standard schedule generation difficulty, LADR exploits the intrinsic spatial locality of visual data. It dynamically rescues high-confidence tokens within resolved neighborhoods using a lightweight confidence margin, employing an inverse scheduling mechanism to adaptively re-align the generation timeline. Extensive evaluations across four publicly popular text-to-image generation benchmarks demonstrate that our method achieves a superior efficiency-fidelity trade-off. It delivers a significant $4\times$ speedup over non-cached baselines and $2\times$ speedup over heuristic-based methods, without model re-training or architectural modifications. Our findings underscore that while text-optimized heuristics provide a foundation, optimal acceleration in the visual domain requires strategies that explicitly respect the 2D spatial structure of the modality, paving the way for plug-and-play DLMs.

\paragraph{Future Work.}
While LADR demonstrates strong efficiency–quality trade-offs for text-to-image diffusion, several directions remain open for future exploration. First, extending locality-aware rescue to temporally structured modalities such as video generation may require jointly modeling spatial and temporal frontiers, where locality spans both space and time. Second, we anticipate that integrating LADR with emerging architectural optimizations (e.g., sparse attention or lightweight distillation) may yield complementary gains, pushing DLMs closer to real-time multimodal generation.

\section*{Limitations}
While LADR demonstrates the potential of exploiting image spatial locality for acceleration of parallel decoding, our current method still has some limitations. The implementation relies on empirically determined hyperparameters, such as the confidence threshold $\tau$ and rescue ratios $\alpha$. These values were selected to validate the core hypothesis that spatial neighbors facilitate faster convergence rather than to locate the global optimum. Furthermore, as a training-free acceleration method, LADR’s performance is influenced by the foundational backbone. To prevent unfaithful generation, existing post-hoc alignment frameworks can be integrated to ensure prompt-image faithfulness. 
\bibliography{custom}

\clearpage
\appendix
\section{Theoretical Proofs and Derivations}
\label{sec:appendix_proofs}

In this section, we provide the detailed mathematical derivations for the propositions and theorems presented in the main methodology.

\subsection{Proof of Theorem \ref{thm:risk_bound} (Margin-based Error Bound)}
\label{proof:thm1}

\textbf{Problem Statement:} 
Let $\mathbf{p} = [p_1, p_2, \dots, p_K]$ be the probability distribution over $K$ classes, sorted such that $p_{(1)} \ge p_{(2)} \ge \dots \ge p_{(K)}$. The predicted class is $\hat{y} = \text{argmax}_k p_k$. The probability of error is $P(\mathcal{E}) = 1 - p_{(1)}$. Given the margin constraint $p_{(1)} - p_{(2)} \ge \tau$, we seek the upper bound of $P(\mathcal{E})$.

\begin{proof}
We aim to find the upper bound for the probability of error $P(\mathcal{E}) = 1 - p_{(1)}$. This is strictly equivalent to finding the minimum possible value for the confidence $p_{(1)}$ subject to the given margin constraints.

1. \textbf{Define Constraints:}
   The probability distribution must sum to 1, and the margin condition must hold:
   \begin{align}
       p_{(1)} + p_{(2)} + \sum_{k=3}^K p_{(k)} &= 1, \label{eq:sum_one} \\
       p_{(1)} - p_{(2)} &\ge \tau \implies p_{(2)} \le p_{(1)} - \tau. \label{eq:margin}
   \end{align}

2. \textbf{Worst-Case Analysis for $p_{(1)}$:}
   To minimize $p_{(1)}$, we must maximize the probability mass assigned to the remaining classes, particularly the closest competitor $p_{(2)}$. The most extreme case occurs when the entire probability mass is concentrated in the top two classes, meaning $\sum_{k=3}^K p_{(k)} \approx 0$. 
   
   Under this worst-case assumption, we have:
   \begin{align}
       p_{(1)} + p_{(2)} \approx 1. \label{eq:top2_sum}
   \end{align}

3. \textbf{Deriving the Lower Bound of Confidence:}
   Substituting the margin constraint from Eq.~\eqref{eq:margin} into Eq.~\eqref{eq:top2_sum}, we obtain:
   \begin{align}
       1 &= p_{(1)} + p_{(2)} \nonumber \\
         &\le p_{(1)} + (p_{(1)} - \tau) \nonumber \\
         &= 2p_{(1)} - \tau.
   \end{align}
   Rearranging the inequality gives the lower bound for $p_{(1)}$:
   \begin{align}
       2p_{(1)} &\ge 1 + \tau \nonumber \\
       p_{(1)} &\ge \frac{1 + \tau}{2}. \label{eq:p1_lower_bound}
   \end{align}

4. \textbf{Calculating the Error Bound:}
   The probability of error is the complement of the top-1 probability. Using Eq.~\eqref{eq:p1_lower_bound}, we bounded the error as:
   \begin{align}
       P(\mathcal{E}) &= 1 - p_{(1)} \nonumber \\
       &\le 1 - \frac{1 + \tau}{2} \nonumber \\
       &= \frac{1 - \tau}{2}.
   \end{align}
   
   This concludes the proof. The error is strictly bounded by a linear function of the threshold $\tau$.
\end{proof}

\subsection{Justification of Proposition~\ref{prop:entropy} (Locality-Induced Information Gain)}
\label{proof:lemma1}

\textbf{Proposition Statement:} The mutual information $I(z_i; \mathbf{z}_{\mathcal{N}(i)})$ dominates $I(z_i; \mathbf{z}_{\mathcal{S}_{dist}})$.

\begin{proof}
We derive this property from the definition of Conditional Mutual Information and the Markov property of Convolutional Neural Networks (CNNs).

1. \textbf{Definition of Information Gain:}
   Let $\Omega$ be the set of all tokens. The information gain for a target token $z_i$ given a subset $\mathcal{S}$ is:
   \begin{align}
       IG(\mathcal{S}) &= I(z_i; \mathcal{S} | \Omega \setminus (\{z_i\} \cup \mathcal{S})) \nonumber \\
       &= H(z_i | \Omega \setminus (\{z_i\} \cup \mathcal{S})) - H(z_i | \Omega \setminus \{z_i\}).
   \end{align}

2. \textbf{Spatial Correlation Decay:}
   For latent codes $z$ derived from a VQ-GAN encoder $E$, the covariance between features at spatial locations $u$ and $v$ generally follows a decay function dependent on Euclidean distance $d(u, v)$:
   \begin{align}
       \text{Cov}(z_u, z_v) \propto \exp\left(-\frac{d(u, v)^2}{2\sigma^2}\right),
   \end{align}
   where $\sigma$ corresponds to the Effective Receptive Field (ERF).

3. \textbf{Entropy and Correlation:}
   For Gaussian-like distributions, the conditional entropy is related to the correlation coefficient $\rho$:
   \begin{align}
       H(z_i | z_j) \approx \frac{1}{2} \log(1 - \rho_{ij}^2) + \text{const}.
   \end{align}
   Higher correlation $\rho_{ij}$ leads to lower conditional entropy $H(z_i | z_j)$.

4. \textbf{Comparing Neighbors vs. Distant Tokens:}
   Let $j \in \mathcal{N}(i)$ be a spatial neighbor and $k \in \mathcal{S}_{dist}$ be a distant token.
   \begin{align}
       d(i, j) &\ll d(i, k) \\
       \Rightarrow \rho_{ij} &> \rho_{ik} \\
       \Rightarrow H(z_i | z_j) &< H(z_i | z_k).
   \end{align}

5. \textbf{Conclusion:}
   Since observing neighbors reduces the conditional entropy more than observing distant tokens:
   \begin{align}
       I(z_i; \mathbf{z}_{\mathcal{N}(i)}) > I(z_i; \mathbf{z}_{\mathcal{S}_{dist}}).
   \end{align}
   Thus, the optimal strategy for variance reduction is to prioritize $\mathcal{N}(i)$.
\end{proof}

\subsection{Derivation of Inverse Scheduling (Proposition~\ref{prop:manifold})}
\label{proof:prop1}

\textbf{Objective:} Find the effective timestep $t_{new}$ such that the model's training distribution matches the current observation density.

\begin{proof}
1. \textbf{Forward Process Definition:}
   The masking probability at time $t$ is given by the schedule function $\gamma(t)$:
   \begin{align}
       q(z_{t,i} = \texttt{[MASK]}) = \gamma(t).
   \end{align}

2. \textbf{Expected Mask Ratio:}
   For a sequence of length $N$, the number of masked tokens $M_t$ follows a Binomial distribution. The expected mask ratio is:
   \begin{align}
       \mathbb{E}\left[\frac{|\mathcal{M}_t|}{N}\right] = \gamma(t).
   \end{align}

3. \textbf{Perturbation via Rescue:}
   The LADR algorithm unmasks a set of tokens $\mathcal{R}$, changing the actual mask ratio to $\rho_{act}$:
   \begin{align}
       \rho_{act} = \frac{|\mathcal{M}_{prev}| - |\mathcal{R}|}{N}.
   \end{align}

4. \textbf{Manifold Alignment:}
   To ensure the input to the denoiser $p_\theta(\mathbf{z}_{0} | \mathbf{z}_{t_{new}})$ is In-Distribution (ID), we require the expected mask ratio at $t_{new}$ to equal the actual current ratio:
   \begin{align}
       \gamma(t_{new}) &= \rho_{act}.
   \end{align}

5. \textbf{Solving for Timestep:}
   Assuming $\gamma(t)$ is monotonic and invertible (e.g., cosine schedule), we apply the inverse function:
   \begin{align}
       t_{new} = \gamma^{-1}(\rho_{act}).
   \end{align}
   This creates the mapping required for the Manifold Consistent Inverse Scheduling.
\end{proof}

\section{Extended Related Work}
\label{app:related_work}

In this section, we provide a detailed elaboration on the development of Masked Discrete Diffusion for image generation and the current landscape of acceleration strategies.

\subsection{Masked Discrete Diffusion for Image Generation}

Discrete Diffusion Language Models (DLMs)~\cite{sahoo2024simple,nie2025large, song2025seed, arriola2025block} have reformulated the generation process as masked modeling within a discretized vector-quantized (VQ) space.
Pioneered by MaskGIT~\cite{chang2022maskgit}, this paradigm utilizes a bidirectional Transformer coupled with a mask-scheduling strategy to enable image synthesis via iterative parallel decoding. Compared to standard continuous diffusion models~\cite{ho2020denoising,rombach2022high}, this formulation significantly curtails the required sampling steps. 
Building upon this foundation, subsequent architectures have rapidly expanded the field: Paella~\cite{rampas2022novel} optimized U-Net backbones with noise-robust objectives, while Muse~\cite{chang2023muse} demonstrated scalability by integrating pre-trained LLMs for enhanced semantic control.

More recently, the field has witnessed a shift towards unified multimodal understanding and generation~\cite{you2025llada,swerdlow2025unified,xin2025lumina,li2025lavida,li2025lavida1,zhou2025draw}. Notably, frameworks like Lumina-DiMOO~\cite{xin2025lumina} and LaVida-O~\cite{li2025lavida} adopt a generalized discrete diffusion approach that treats visual and textual tokens as a shared sequence. While facilitating versatile generative capabilities across modalities, the iterative mask recovery process still imposes a non-negligible computational overhead. This underscores the necessity for efficient acceleration strategies that can expedite inference without compromising generative integrity.

\subsection{Acceleration of Masked Discrete Diffusion}

The efficacy of discrete diffusion~\cite{wang2025alternate} hinges on iterative refinement, where multiple forward passes resolve the joint distribution of tokens. Unlike autoregressive models that benefit from causal masking and KV-caching~\cite{li2024snapkv, bai2023qwen, guo2025deepseek,cai2024pyramidkv}, masked diffusion relies on bidirectional attention with dynamically shifting mask states~\cite{sahoo2024simple, xin2025lumina}. This characteristic precludes the reuse of historical computations, creating a distinct latency bottleneck.

\paragraph{Distillation-Based Approaches.}
To alleviate latency, research has gravitated towards model distillation~\cite{hinton2014distilling, SongD0S23, deschenaux2025beyond,yin2024one,hayakawa2025distillation}. While Consistency Models~\cite{SongD0S23,kim2024consistency} are effective in continuous pixel space, adapting them to discrete VQ space requires specialized formulations due to the absence of explicit ODE trajectories~\cite{zhu2025di}. Works such as DiMO~\cite{zhu2025di} and Soft-DiMO~\cite{zhu2025soft} address this using policy gradients and soft embedding relaxations to compress multi-step trajectories. However, these methods necessitate computationally expensive re-training and student-teacher alignment, limiting their plug-and-play applicability.

\paragraph{Training-Free Heuristics.}
Advancements in DLMs have also explored architectural optimizations~\cite{wu2025fast2, hu2025accelerating, wu2025fast,wang2025diffusion} and adaptive sampling~\cite{li2025diffusion, israel2025accelerating} primarily for text generation.
While effective for 1D sequences, their direct adaptation to the visual domain is non-trivial. The inherent gap between the sequential dependencies of text and the 2D spatial correlations of images renders text-optimized heuristics suboptimal for visual generation. This discrepancy highlights the need for acceleration strategies explicitly tailored to the spatial redundancy and structural properties of images.

\end{document}